\documentclass[10pt,twocolumn,letterpaper]{article}

\usepackage{cvpr}      

\usepackage{booktabs}
\usepackage{multicol}
\usepackage{colortbl}
\usepackage{multirow}
\usepackage{amsthm}
 
\newtheorem{lemma}{Lemma} 

\usepackage{amssymb}
\usepackage{pifont}
\usepackage{comment}
\usepackage[normalem]{ulem}
\usepackage{xcolor}
\usepackage{algorithmic}
\usepackage{algorithm}
\usepackage{wrapfig}
\usepackage{amsthm,amsmath,amssymb,,bm,bbm}
\usepackage{mdframed}
\usepackage{multicol} 
\usepackage{bbding}
\theoremstyle{plain}
\theoremstyle{definition}
\usepackage{makecell}
\newmdtheoremenv{prop}{Proposition}

\newmdtheoremenv{definition}[theorem]{Definition}

\theoremstyle{remark}
\newmdtheoremenv{corollary}{Corollary}[theorem]
\definecolor{cvprblue}{rgb}{0.21,0.49,0.74}
\usepackage[pagebackref,breaklinks,colorlinks,allcolors=cvprblue]{hyperref}

\def\paperID{14193} 
\def\confName{CVPR}
\def\confYear{2026}

\title{OTPrune: Distribution-Aligned Visual Token Pruning via Optimal Transport}

\author{Xiwen Chen$^{1,2}$\thanks{\textit{These authors contributed equally to this paper.}} \enspace Wenhui Zhu$^{3*}$\thanks{\textit{Now at LinkedIn.}} \enspace  Gen Li$^{2}$ \enspace Xuanzhao Dong$^{3}$ \enspace Yujian Xiong$^{3}$ \enspace  Hao Wang$^{2}$ \\  Peijie Qiu$^{4}$  
\enspace  Qingquan Song$^{5}$ \enspace Zhipeng Wang$^{6\dagger}$\thanks{Corresponding author: \texttt{zhipeng.wang@alumni.rice.edu}}  \enspace Shao Tang$^{7}$ \enspace 
  Yalin Wang$^{3}$ \enspace Abolfazl Razi$^{2}$\thanks{Corresponding author: \texttt{arazi@clemson.edu}}\\
$^{1}$ Morgan Stanley,
$^{2}$ Clemson University,  
$^{3}$ Arizona State University, \\
$^{4}$ Washington University in St. Louis, 
$^{5}$ Texas A\&M University, \\
$^{6}$ Rice University 
$^{7}$ Florida State University
}

\begin{document}
\maketitle

\begin{abstract}
Multi-modal large language models (MLLMs) achieve strong visual–language reasoning but suffer from high inference cost due to redundant visual tokens. 
Recent work explores visual token pruning to accelerate inference, while existing pruning methods overlook the underlying distributional structure of visual representations. We propose OTPrune, a training-free framework that formulates pruning as distribution alignment via optimal transport (OT). By minimizing the 2-Wasserstein distance between the full and pruned token distributions, OTPrune preserves both local diversity and global representativeness while reducing inference cost. Moreover, we derive a tractable submodular objective that enables efficient optimization, and theoretically prove its monotonicity and submodularity, providing a principled foundation for stable and efficient pruning. We further provide a comprehensive analysis that explains how distributional alignment contributes to stable and semantically faithful pruning. Comprehensive experiments on wider benchmarks demonstrate that OTPrune achieves superior performance–efficiency tradeoffs compared to state-of-the-art methods. The code is available at \url{https://github.com/xiwenc1/OTPrune}.

\end{abstract}

\begin{figure}
    \centering
    \includegraphics[width=0.8\linewidth]{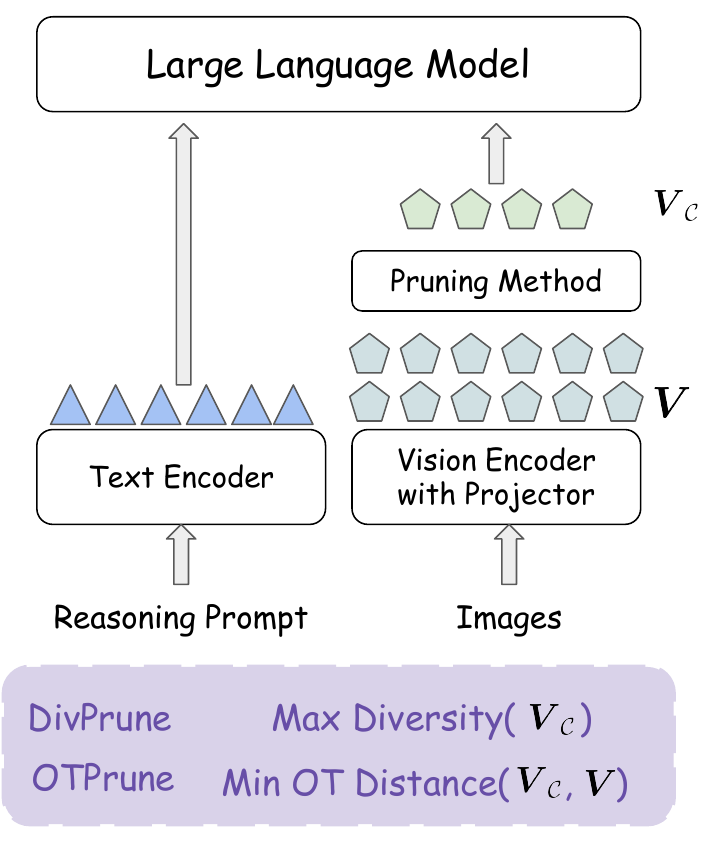}
    \caption{
\textbf{Overview of the proposed OTPrune framework.}
Given visual and textual inputs, a vision encoder with a projector produces image tokens that are fed into a large language model together with text tokens and reasoning prompts.
Existing methods, such as \textit{DivPrune}, select tokens by maximizing diversity, which may overlook global representativeness.
In contrast, \textit{OTPrune} formulates pruning as \emph{distribution alignment} by minimizing the 2-Wasserstein (optimal transport) distance between the full and pruned token distributions, thereby preserving both local diversity and global structure for efficient and semantically faithful multimodal reasoning.
}
    \label{fig:diff}
\end{figure}

\section{Introduction}
Large multimodal models (LMMs) have emerged as a central paradigm for integrating visual and linguistic understanding~\cite{clip,llava_1.5,llavanext,llavanextvideo}.
A prominent subset of these models, known as multimodal large language models (MLLMs), extends LLMs by incorporating visual encoders and cross-modal alignment modules.
Recent representatives such as LLaVA~\cite{llava_1.5,llavanext}, GPT-4V~\cite{openai2023gpt4}, and Gemini~\cite{gemini} have demonstrated remarkable capabilities in visual question answering, image captioning, and open-ended reasoning~\cite{videollama2,lin2023video}.
These models typically transform both textual and visual inputs into token sequences that are jointly processed by a transformer-based decoder~\cite{vaswani2017attention,alayrac2022flamingo}.
However, the number of visual tokens in MLLMs often far exceeds that of textual ones.
A single image can generate hundreds of patch-level tokens~\cite{dosovitskiy2020vit}, which greatly increases the overall sequence length.
Since the self-attention operation in transformers scales quadratically with sequence length~\cite{vaswani2017attention}, longer multimodal contexts, especially in multi-turn reasoning or dialogue scenarios, lead to substantially higher computational and memory costs during inference~\cite{kwon2023vllm,cachedattention2024,soma2025}.

A growing body of work aims to mitigate this bottleneck through visual token pruning, which removes redundant vision tokens while preserving task-relevant semantics~\cite{kong2025token}.
Empirical studies reveal that visual tokens contain redundancy, allowing a large portion of tokens (e.g., 70–90\%) to be pruned with little impact on accuracy~\cite{fastv,PruMerge,vtw}. A variety of pruning strategies have been proposed to reduce the inference cost of large multimodal models (LMMs), including attention-based~\cite{fastv,PruMerge,guo2024attention,lin2025boosting}, calibration-based~\cite{FitPrune,vtw}, and trainable approaches~\cite{m3,li2025tokenpacker}.
Among them, DivPrune~\cite{divprune} takes a notable step forward by introducing a diversity-driven criterion that favors dissimilar tokens, effectively reducing redundancy and improving generalization.
However, diversity alone is not sufficient to ensure representativeness. \textit{While DivPrune enforces separation among selected tokens, it neglects the global relationship between the pruned subset and the full token distribution.} Consequently, the retained subset may fail to capture the underlying covariance structure or semantic coverage of the original vision tokens, especially under high pruning ratios. In other words, diversity promotes dissimilarity within the subset but does not guarantee that the subset faithfully approximates the full feature manifold.

To address this, we argue that effective pruning should account for both local diversity and global distributional alignment.
Instead of merely selecting distant tokens, we aim to retain those that best represent the geometry of the full token space.
To this end, we propose \textit{OTPrune}, a principled and training-free framework that formulates token pruning as a distribution approximation problem based on \textit{optimal transport (OT) problem}.
The selected subset is required to induce a distribution that closely matches that of all tokens, achieved by minimizing the \textit{2-Wasserstein distance} between their empirical distributions.
This formulation explicitly aligns the geometric and statistical structures of tokens in the embedding space, ensuring both diversity and representativeness. An overview of the framework is shown in Fig.~\ref{fig:diff}.

However, directly minimizing the Wasserstein distance over discrete subsets is computationally intractable. To solve this, we approximate the token distributions with Gaussian surrogates that match their covariance statistics, leading to a closed-form expression for the 2-Wasserstein distance. We then show that the optimization objective can be effectively optimized by maximizing the log-determinant of the covariance alignment matrix between the selected and full tokens, which serves as a tractable lower-bound surrogate of the trace objective, i.e., \textit{second-order OT surrogate}.
The resulting criterion is monotone and submodular, enabling efficient greedy maximization via Cholesky decomposition with polynomial complexity. In summary, our contributions are as follows:
\begin{itemize}
    \item We are the first to formulate vision token pruning as a distribution alignment problem by minimizing the 2-Wasserstein distance between the full and pruned token distributions, bridging the gap between local diversity and global representativeness. Empirical analysis confirms that this formulation yields more semantically faithful pruning behavior.

\item We derive a tractable \textit{second-order OT surrogate} of the Wasserstein distance that yields a tractable log-determinant objective, which is provably monotone and submodular. Empirical results on synthetic data further demonstrate that this surrogate closely aligns with the true optimal transport objective, providing both theoretical grounding and practical fidelity.

\item Comprehensive empirical evaluation and analysis.
Across diverse image and video benchmarks, OTPrune achieves superior accuracy–efficiency trade-offs compared to state-of-the-art methods such as FastV, VTW, and DivPrune. 
\end{itemize}

\section{Related Works}
\subsection{Multi-modal large language models (MLLMs)}
Multimodal large language models (MLLMs) have recently emerged as a general framework for understanding and reasoning across different modalities, including text, image, audio, and video~\cite{llava_1.5,llavanext,llavanextvideo,lin2023video,videollama2,gpt40,gemini}.
They extend large language models (LLMs) with visual or auditory encoders, allowing them to process multimodal inputs in a unified token space.
Such models have shown strong performance on tasks like visual question answering, image and video captioning, and visual grounding, as well as more structured tasks such as segmentation and visual editing~\cite{wei2025instructseg,zhang2025instructvedit}.

\subsection{Efficient MLLMs}
Recent studies have explored various strategies to improve the inference efficiency of LMMs. 
Architectural optimization methods include replacing transformer backbones with more efficient state-space models such as Mamba~\cite{vlmamba, mamba}, retraining with smaller-scale language models~\cite{tinygptv, tinyllava}, and applying knowledge distillation for model compression~\cite{llavadi}. 
At the inference level, computation skipping~\cite{shukor2024skip} and efficient decoding techniques like speculative decoding~\cite{speculative2024} have been proposed to accelerate reasoning while preserving performance. 
However, these methods require retraining or architectural changes, which limit their adaptability to general MLLMs. 
Alternatively, visual token pruning~\cite{fastv,vtw,zhu2025evtp} improves efficiency by removing redundant visual tokens from pretrained LMMs without altering the architecture, motivating our scalable and lightweight design.

\subsection{Visual Token Pruning}
Visual token pruning has become an effective way to reduce the inference cost of large multimodal models (LMMs). Early studies mainly rely on attention-based pruning, where tokens with low attention responses are removed. PruMerge~\cite{PruMerge} and FastV~\cite{fastv} follow this idea, but later analyses~\cite{guo2024attention,lin2025boosting} show that attention magnitude does not always reflect true token importance, especially under high pruning ratios.
Another direction is calibration-based pruning, where pruning ratios or layers are decided using a small calibration set, as in FitPrune~\cite{FitPrune} and VTW~\cite{vtw}. FitPrune averages attention maps across samples to build dataset-level statistics, but attention is inherently input-dependent, making these statistics unstable and less generalizable to unseen data. Such methods rely on global heuristics rather than modeling instance-level structure, which limits their robustness across inputs.
Fine-tuning-based approaches, such as M³~\cite{m3} and TokenPacker~\cite{li2025tokenpacker}, improve adaptivity by learning hierarchical or compact token representations, though at significant computational cost. More recently, DivPrune~\cite{divprune} enhances pruning by encouraging token diversity, yet it focuses on local dissimilarity and overlooks global distributional alignment, which is crucial for maintaining representativeness.
These observations suggest that effective pruning should balance local diversity and global distributional structure without relying on extra calibration or training. Motivated by this, we introduce an optimal transport–based pruning framework that preserves the global feature distribution while removing redundant tokens. This design enables efficient and semantically faithful pruning, and we further provide theoretical analysis that explains why the optimal transport formulation yields a stable and principled token selection process.

\section{Method}

\subsection{Preliminary} 
Multimodal inputs typically generate highly redundant token sequences. For a text input, the encoder produces a sequence $\boldsymbol{T} = \{t_1, \dots, t_n\}$ while for a visual input the encoder produces
$\boldsymbol{V} = \{v_1, \dots, v_m\}\in \mathbb{R}^{m \times d}$.
Although $n$ and $m$ may be large, many tokens contribute only marginally to the semantic meaning required for the downstream task. Transmitting all tokens thus wastes valuable communication resources. We can formulate visual token selection as follows. Specifically, we aim to select subsets $\mathcal{C} \subseteq [m]$ such that the selected tokens $\boldsymbol{V}_\mathcal{C}\in \mathbb{R}^{k \times d}$ form semantically representative summaries of their original sequences, where $k$ is the computation budget, i.e., the maximum number of tokens that can be selected. , and $d$ denotes the latent dimensions.

Formally, let $f(\cdot)$ denote the downstream LLM task model and $\mathcal{L}(\cdot,\cdot)$ denote a task-specific loss function (e.g., cross-entropy for classification or semantic distance for generation). The vision token pruning problem can be written as:
\begin{align}\label{eq:ori_problem}
\min_{\mathcal{C}} \; \mathcal{L}\!\left(f(\boldsymbol{T}, \boldsymbol{V}_\mathcal{C}), \; f(\boldsymbol{T}, \boldsymbol{V})\right) 
\quad \text{s.t.} \quad  |\mathcal{C}| = k.
\end{align}

\subsection{Pruning as OT Problem}
While Eq.~\ref{eq:ori_problem} formulates pruning in terms of the downstream task loss, such loss is typically infeasible to compute during inference. 
To enable a task-agnostic formulation, we instead cast pruning as the problem of approximating the distribution of all tokens using that of a pruned subset. 
We hypothesize that \textit{a closer approximation of the full distribution leads to better downstream task performance.}

Let $P$ denote the distribution induced by the full token set and $Q$ the distribution induced by the selected subset. 
We assume $Q$ is defined as the empirical distribution of the selected tokens with uniform weights ($\tfrac{1}{k}$ per token), ensuring that its total mass is normalized to 1. 
Our goal is to find a subset $\mathcal{C}$ such that $Q$ closely matches $P$ under an appropriate discrepancy measure. 
Specifically, we focus on minimizing the squared 2-Wasserstein distance:
\begin{align}\label{eq:w1}
\min_Q \mathcal{W}_2^2(P,Q) = \inf_{\pi \in \Pi(P,Q)} \mathbb{E}_{(x,y)\sim \pi}\!\left[\|x - y\|_2^2\right],
\end{align}
which measures the geometric discrepancy between distributions.

\noindent\textbf{Hypothesis Validation.} 
To validate our hypothesis that a better distributional approximation leads to improved downstream performance, 
we conduct a preliminary study using LLaVA~1.5-7B across 11 multimodal reasoning and understanding benchmarks (see Sec.~\ref{sec:exp_setup} for details of datasets). 
Given a sequence of $m$ tokens, we sample $k$ tokens ($k < m$) using several simple index-based strategies. 
\textit{First-K sampling} selects the first $k$ tokens from the sequence, retaining the earliest portion. 
\textit{Last-K sampling} keeps the last $k$ tokens, focusing on the most recent part of the sequence. 
\textit{Uniform index sampling} evenly samples from the index range, ensuring that the selected $k$ tokens are uniformly distributed across the entire sequence. 
We also include \textit{DivPrune}~\cite{divprune}, the recent state-of-the-art pruning method, and random selection for comparison. For each strategy, we compute two rankings: (i) the rank of the OT distance between the selected subset and the full token set (lower rank indicates closer distributional alignment), and (ii) the rank of the downstream task performance (higher accuracy corresponds to a lower rank). 
We then calculate Spearman's rank correlation between these two rankings. 
As shown in Fig.~\ref{fig:draw_motivation}, the strong positive correlation confirms that subsets achieving smaller OT distances tend to yield higher downstream performance, supporting our distribution-alignment hypothesis.

\begin{figure*}
    \centering
    \includegraphics[width=1\linewidth]{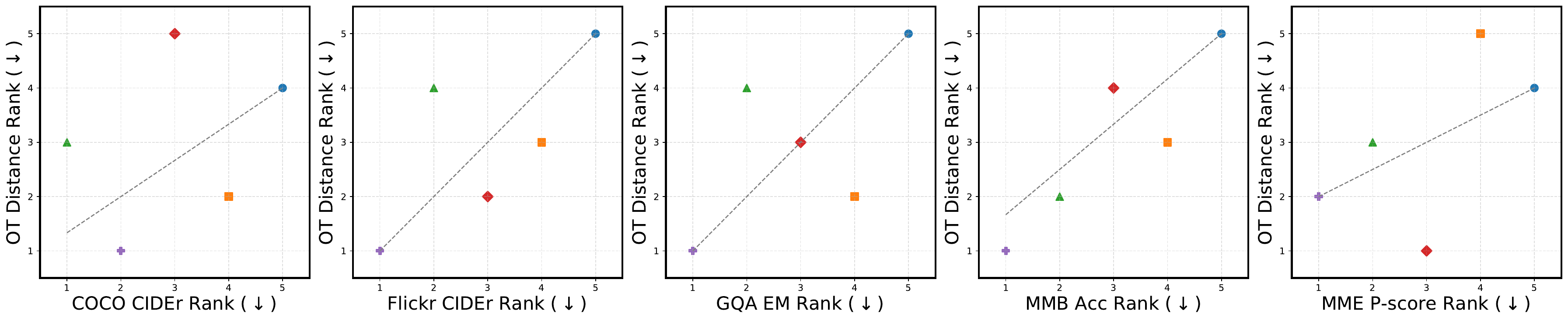}
    \caption{
\textbf{Correlation between OT distance and downstream performance.} 
We evaluate several manually designed selection strategies (\textit{First-K}, \textit{Last-K}, \textit{Uniform}, and \textit{Random}) along with the SOTA method \textit{DivPrune}~\cite{divprune}. 
For each strategy, we compute the OT distance between the selected subset and the full token set and measure downstream performance across 11 multimodal benchmarks. 
Both OT distance (lower distance $\Rightarrow$ lower rank) and task performance (higher score $\Rightarrow$ lower rank) are ranked to compute Spearman’s correlation. 
A strong positive correlation confirms that smaller OT distance corresponds to higher performance. 
}

    \label{fig:draw_motivation}
\end{figure*}

\subsection{OTPrune Selector}\label{sec:OTselector}
For analytic tractability, we approximate the OT distance by a Gaussian surrogate. Motivated by moment-matching approaches~\cite{chen2020homm}, we adopt a second-order approximation. Following the analytical convenience in~\cite{yu2020learning}, we use zero-mean Gaussians to represent the true token distributions: $P \approx \mathcal{N}(0, \Sigma)$ and $Q \approx \mathcal{N}(0, \Sigma_{\mathcal{C}})$. 
In this case, the squared 2-Wasserstein distance in Eq.~\ref{eq:w1} admits a closed form:
\begin{align}\label{eq:w2}
\mathcal{W}_2^2(P,Q) 
= \operatorname{tr}(\Sigma) + \operatorname{tr}(\Sigma_\mathcal{C}) 
- 2\,\operatorname{tr}\!\Big((\Sigma^{1/2}\Sigma_\mathcal{C}\Sigma^{1/2})^{1/2}\Big).
\end{align}
To avoid scaling bias and ensure fair comparison across features, we normalize each embedding dimension to unit variance. This yields $\Sigma_{ii}=1$ and $\operatorname{tr}(\Sigma) = d$, leaving the distance primarily sensitive to mismatches in covariance structure rather than absolute scaling. Equivalently, we are maximizing the following objective function $f(\mathcal{C})$:
\begin{align}\label{eq:w3}
    \max_{\mathcal{C}} ~ \Big(\underbrace{\operatorname{tr}(\Sigma^{1/2}\Sigma_\mathcal{C}\Sigma^{1/2})^{1/2}}_{f(\mathcal{C})}\Big)^2\quad \text{s.t.} \quad  |\mathcal{C}| = k.
\end{align}
Although $\mathcal{W}_2^2$ is a principled distortion metric, it is still computationally expensive to optimize over discrete subsets.

To this end, instead of directly optimizing the original objective, we consider a more \textit{tractable lower bound}. 
Specifically, we introduce the $\gamma$-log-determinant operator for a positive semidefinite square matrix $\boldsymbol{X}$:
\begin{equation}
\Psi(\boldsymbol{X}) = \log\det(\boldsymbol{I} + \gamma \boldsymbol{X})
= \sum_{i=1}^{d} \log(1 + \gamma \lambda_i),
\end{equation}
where $\{\lambda_i\}$ are the eigenvalues of $\boldsymbol{X}$ and $\gamma > 0$ is a fixed scaling parameter. 
Since $\boldsymbol{X}$ is positive semidefinite (so $\lambda_i \ge 0$), we have the following inequalities (\textcolor{Periwinkle}{\textbf{See Appendix A.1 for the proof}}),
\begin{equation}\label{eq:approx}
\Psi(\boldsymbol{X}) \le \gamma\,\operatorname{tr}(\boldsymbol{X})\leq \gamma (\operatorname{tr}(\boldsymbol{X}^{1/2}))^2,
\end{equation}
indicating that $\Psi(\boldsymbol{X})$ serves as a lower bound of the trace objective. 
For small $\gamma$, this bound becomes tight because $\log(1+\gamma \lambda_i) \approx \gamma \lambda_i$ when $x$ is close to zero, leading to $\Psi(\boldsymbol{X}) \approx \gamma\,\operatorname{tr}(\boldsymbol{X})$. Thus, since $\Sigma^{1/2}\Sigma_\mathcal{C}\Sigma^{1/2} \succeq 0$, 
we can safely apply $\boldsymbol{X}=\Sigma^{1/2}\Sigma_\mathcal{C}\Sigma^{1/2}$ to the above inequality and obtain the following surrogate objective\footnote{We deliberately avoid using 
$\boldsymbol{X}=(\Sigma^{1/2}\Sigma_{\mathcal{C}}\Sigma^{1/2})^{1/2}$ 
to compute $\Psi(\boldsymbol{X})$ because evaluating matrix square roots
(e.g., via SVD or eigendecomposition) would incur significant overhead 
during greedy subset selection. Instead, we adopt 
$\boldsymbol{X}=\Sigma^{1/2}\Sigma_{\mathcal{C}}\Sigma^{1/2}$,
which yields a tractable surrogate with the same monotonicity and 
submodularity properties, while preserving the lower-bound relationship 
in Eq.~\ref{eq:approx}.}:
\begin{align}\label{eq:new_objective}
   \max_{\mathcal{C}} \;
   \underbrace{\log\det\!\big(\boldsymbol{I} + \gamma\,\Sigma^{1/2}\Sigma_\mathcal{C}\Sigma^{1/2}\big)}_{\tilde{f}(\mathcal{C})}.
\end{align}

 Since $\Sigma\approx \frac{\boldsymbol{V}^T\boldsymbol{V}}{m},\Sigma_\mathcal{C}\approx \frac{\boldsymbol{V}_\mathcal{C}^T\boldsymbol{V}_\mathcal{C}}{k}$, and by leveraging the \textit{Sylvester’s determinant identity} $\log \det\!\big(I + \boldsymbol{X}\boldsymbol{X}^T\big)= \log \det\!\big(I + \boldsymbol{X}^T\boldsymbol{X}\big)$ for any matrix $\boldsymbol{X}$~\cite{petersen2008matrix},
we can rewrite Eq.~\ref{eq:new_objective} as 
\begin{align}\label{eq:kernel}
    &\log \det\!\big(I + \gamma \Sigma^{1/2}\Sigma_\mathcal{C}\Sigma^{1/2}\big)  \\ \nonumber
    \approx& \log \det\!\big(I + \frac{\gamma}{mk} (\boldsymbol{V}^T\boldsymbol{V})^{1/2}\boldsymbol{V}_\mathcal{C}^T\boldsymbol{V}_\mathcal{C}(\boldsymbol{V}^T\boldsymbol{V})^{1/2}\big)  \\ \nonumber
     =&  \log \det\!\big(I + \frac{\gamma}{mk} \boldsymbol{V}_\mathcal{C}(\boldsymbol{V}^T\boldsymbol{V})^{1/2}(\boldsymbol{V}^T\boldsymbol{V})^{1/2}\boldsymbol{V}_\mathcal{C}^T\big) \\ \nonumber
     =&\log \det\!\big(I + \frac{\gamma}{mk} \boldsymbol{V}_\mathcal{C}\boldsymbol{V}^T\boldsymbol{V}\boldsymbol{V}_\mathcal{C}^T\big) 
     \\ \nonumber
     =& \log \det\!\big(I + \frac{\gamma}{mk} \boldsymbol{\tilde{V}}_\mathcal{C}\boldsymbol{\tilde{V}}_\mathcal{C}^T\big),
\end{align}
where $\boldsymbol{\tilde{V}}_\mathcal{C}=\boldsymbol{V}_\mathcal{C}\boldsymbol{V}^T$ and $\tilde{\gamma}= \frac{\gamma}{mk}$.
Therefore, now we are targeting to maximize $\log \det\!\big(I + \tilde{\gamma}\boldsymbol{V}_\mathcal{C}\boldsymbol{V}^T\boldsymbol{V}\boldsymbol{V}_\mathcal{C}^T\big)$, subject to $ |\mathcal{C}| = k$. Notably, this objective defines a monotone submodular set function 
with respect to $\mathcal{C}$ (\textcolor{Periwinkle}{\textbf{See Appendix A.2 for the proof}}). Due to optimizing this objective being NP-hard, we therefore adopt a greedy approximation strategy based on Cholesky decomposition~\cite{chen2018fast}. 
Under this strategy, the $t$-th greedy iteration updates the
$(m-t)$ remaining items, each requiring an $O(t)$ inner product
between the current Cholesky coefficients, resulting in a per-iteration
cost of $O((m-t)t)$. Summing over $t=1,\dots,k$ gives a total greedy
complexity of $O(mk^2)$.
Formally, we give the algorithmic summary in Algorithm~\ref{alg:GOT-solver}. 
\noindent\textbf{Theoretical Bound.} Due to the submodular nature of the objective function defined in Eq.~\ref{eq:approx}, the greedy selection procedure enjoys the classical $(1 - 1/e)$ approximation worst-case guarantee with respect to the optimal subset~\cite{krause2007near}:
\begin{equation}
    \tilde{f}(\mathcal{C}_{\text{greedy}}) \ge (1 - e^{-1})\, \tilde{f}(\mathcal{C}^{*}),
\end{equation}
where $\tilde{f}(\mathcal{C})$ denotes the objective value for a subset $\mathcal{C}$ computed from Eq.~\ref{eq:approx}, $\mathcal{C}_{\text{greedy}}$ is the subset selected by our algorithm, and $\mathcal{C}^{*}$ is the optimal subset. Since Eq.~\ref{eq:approx} serves as a close surrogate to the original objective in Eq.~\ref{eq:w3}, the same approximation bound is expected to hold in practice. Moreover, as demonstrated in Section~\ref{sec:analysisOfAlg}, our empirical analysis shows that the greedy algorithm often achieves a substantially higher optimality ratio than the theoretical $(1 - 1/e)$ bound.

\begin{algorithm}[h]
\caption{OTPrune (\textcolor{Periwinkle}{\textbf{See Appendix B for details}})} 
\label{alg:GOT-solver}
\begin{algorithmic}[1]
\REQUIRE Vision tokens $\boldsymbol{V} \in \mathbb{R}^{m \times d}$, size $k$, scalar $\tilde{\gamma}$
\ENSURE Selected index set $\mathcal C$
\STATE Precompute $\boldsymbol{\tilde{V}} = \boldsymbol{V}\boldsymbol{V}^\top $, 
with rows $\boldsymbol{w}_i^\top$ for $i=1,\dots,m$, i.e. $\boldsymbol{w}_i = \boldsymbol{\tilde{V}}[i,:]^\top$
\STATE Initialize $\mathcal C \gets \emptyset$, 
$d_i^2 \gets 1+\tilde{\gamma}\|\boldsymbol{w}_i\|_2^2$ for all $i$, 
and $\boldsymbol{c}_i \gets \emptyset$
\WHILE{$|\mathcal C| < k$}
    \STATE $j \gets \arg\max_{i \notin \mathcal C} d_i^2$
    \STATE $\mathcal C \gets \mathcal C \cup \{j\}$
    \FOR{each $i \notin \mathcal C$}
        \STATE $e_i \gets 
        \dfrac{\tilde{\gamma}\,\langle \boldsymbol{w}_j, \boldsymbol{w}_i\rangle 
        - \langle \boldsymbol{c}_j, \boldsymbol{c}_i\rangle}{d_j}$
        \STATE $\boldsymbol{c}_i \gets [\boldsymbol{c}_i \;\; e_i]$
        \STATE $d_i^2 \gets d_i^2 - e_i^2$
    \ENDFOR
\ENDWHILE
\RETURN $\mathcal C$
\end{algorithmic}
\end{algorithm}

\setlength\heavyrulewidth{0.25ex}
\begin{table*}[tb!]
\centering
\small

\caption{
\textbf{Comparison of \textit{OTPrune} with diverse baselines across 11 multimodal benchmarks.} DivPrune$^{*}$ aligns its token retention rate with PruMerge. For fair evaluation, all one-shot methods strictly retain 9.8\% of visual tokens, using TFLOPs to indicate computational efficiency. The \textcolor{red}{\textit{Avg. Rank $\downarrow$}} assesses all evaluated methods, while the \textcolor{red}{\textit{One-shot Rank $\downarrow$}} strictly compares one-shot techniques. Our \textit{OTPrune} consistently delivers the best performance. (${\bullet}$: requires finetuning; ${\bigtriangleup}$: requires calibration).
}
\vspace{3pt}

\resizebox{\textwidth}{!}{%
\begin{tabular}{c|cc|ccccccccccc|cc}
\toprule
\multirow{2}{*}{} & \multirow{2}{*}{Method} & {TFLOP} & COCO  & Flickr & GQA & MMB & MME & MMMU & Nocaps & OKVQA & POPE & SQA  & SEEDB & {\textbf{Avg.}} & {\textbf{One-shot}} \\ 
    &          &   (ratio~\%)   & CIDEr & CIDEr  & EM    & Acc & P-score & Acc   & CIDEr  & EM    & F1 & EM    & Acc & \textbf{Rank}$\downarrow$ & \textbf{Rank} $\downarrow$ \\ 
\midrule

\multirow{15}{*}{\rotatebox[origin=c]{90}{\textbf{LLaVA 1.5-7B}}} 

& \textbf{Original}  
& \cellcolor[gray]{0.9}3.228 (100.00) 
& \cellcolor[gray]{0.9}1.10  & \cellcolor[gray]{0.9}0.75 & \cellcolor[gray]{0.9}61.96 & \cellcolor[gray]{0.9}64.09 
& \cellcolor[gray]{0.9}1506 & \cellcolor[gray]{0.9}36.44 
& \cellcolor[gray]{0.9}1.06 & \cellcolor[gray]{0.9}53.39 & \cellcolor[gray]{0.9}85.84 
& \cellcolor[gray]{0.9}69.41 & \cellcolor[gray]{0.9}66.17 & \cellcolor[gray]{0.9}-- & \cellcolor[gray]{0.9}-- \\[2pt]

& \multicolumn{15}{l}{\hspace{0.3em}\textcolor{gray}{\scriptsize\textit{Category: One-shot Pruning Methods}}} \\[-2pt]
& \hspace{0.8em}VTW~\cite{vtw} & 0.603 (18.46) & 0.05 & 0.03 & 38.94 & 21.31 & 681 & 32.60 & 0.03 & 18.64 & 25.35 & 65.29 & 36.13 & 8.23 & 3.41\\
& \hspace{0.8em}FastV~\cite{fastv} & 0.514 (15.69) & 0.06 & 0.03 & 38.73 & 20.62 & 696 & 32.00 & 0.04 & 18.32 & 32.84 & 65.15 & 35.69 & 8.41 & 3.59\\
& \hspace{0.8em}DivPrune & 0.512 (15.63) & 0.96 & 0.62 & \textbf{56.85} & 59.19 & 1328 & \textbf{35.89} & 0.92 & 46.98 & \textbf{86.02} & 68.27 & \textbf{59.47} & 3.23 & 1.64\\
& \hspace{0.8em}\textbf{OTPrune (Ours)} & 0.512 (15.63) 
& \textbf{1.00} & \textbf{0.68} & 56.62 & \textbf{60.40} & \textbf{1368} & 35.44 
& \textbf{0.99} & \textbf{47.09} & 79.59 & \textbf{68.42} & 59.31 & \textcolor{red}{\textbf{2.36}} & \textcolor{red}{\textbf{1.36}} \\[3pt]

& \multicolumn{15}{l}{\hspace{0.3em}\textcolor{gray}{\scriptsize\textit{Category: Adaptive Pruning Methods}}} \\[-2pt]
& \hspace{0.8em}PruMerge~\cite{PruMerge} & Variable & 0.77 & 0.50 & 51.30 & 54.47 & 1259 & 35.11 & 0.73 & 41.74 & 66.89 & 68.91 & 53.26 & 6.09 & --\\
& \hspace{0.8em}DivPrune$^{*}$ & Variable & 0.91 & 0.56 & 55.25 & 58.16 & 1330 & 35.44 & 0.87 & 44.38 & 83.06 & 67.87 & 57.88 & 4.73 & --\\[2pt]

& \multicolumn{15}{l}{\hspace{0.3em}\textcolor{gray}{\scriptsize\textit{Category: Finetuning-based Methods}}} \\[-2pt]
& \hspace{0.8em}$\text{FitPrune}^{\bigtriangleup}$~\cite{FitPrune} & 0.513 (15.65) & 0.90 & 0.56 & 52.39 & 57.65 & 1197 & 36.00 & 0.86 & 42.53 & 60.89 & 68.02 & 54.84 & 5.59 & --\\
& \hspace{0.8em}$\text{M}^{3\bullet}$~\cite{m3} & 0.512 (15.63) & 1.00 & 0.67 & 60.81 & 65.81 & 1391 & 31.80 & 0.95 & 55.12 & 86.33 & 64.65 & 64.93 & 2.68 & --\\
& \hspace{0.8em}$\text{PruMerge-LoRA}^{\bullet}$ & Variable & 0.96 & 0.63 & 55.96 & 59.88 & 1334 & 34.89 & 0.90 & 47.99 & 77.13 & 68.32 & 57.93 & 3.68 & --\\
\midrule

\multirow{10}{*}{\rotatebox[origin=c]{90}{\textbf{LLaVA 1.5-13B}}} 

& \textbf{Original} 
& \cellcolor[gray]{0.9}6.281 (100.00) 
& \cellcolor[gray]{0.9}1.16 & \cellcolor[gray]{0.9}0.80 & \cellcolor[gray]{0.9}63.33 & \cellcolor[gray]{0.9}68.64 
& \cellcolor[gray]{0.9}1522 & \cellcolor[gray]{0.9}35.67 
& \cellcolor[gray]{0.9}1.09 & \cellcolor[gray]{0.9}58.28 & \cellcolor[gray]{0.9}85.99 
& \cellcolor[gray]{0.9}72.88 & \cellcolor[gray]{0.9}66.82 & \cellcolor[gray]{0.9}-- & \cellcolor[gray]{0.9}-- \\[2pt]

& \multicolumn{15}{l}{\hspace{0.3em}\textcolor{gray}{\scriptsize\textit{Category: One-shot Pruning Methods}}}\\[-2pt]
& \hspace{0.8em}VTW~\cite{vtw} & 1.030 (16.16) & 0.08 & 0.05 & 39.71 & 21.91 & 622 & 32.10 & 0.05 & 22.49 & 0.40 & 66.24 & 38.59 & 6.00 & 4.00\\
& \hspace{0.8em}FastV~\cite{fastv} & 1.003 (15.73) & 0.38 & 0.18 & 44.98 & 37.80 & 942 & 35.11 & 0.33 & 32.14 & 30.02 & 69.96 & 44.95 & 4.82 & 2.91\\
& \hspace{0.8em}DivPrune & 1.002 (15.71) & 1.00 & 0.66 & \textbf{57.29} & 63.40 & 1407 & 34.89 & 0.95 & 53.29 & \textbf{83.43} & 72.34 & 62.04 & 2.14 & 1.91\\
& \hspace{0.8em}\textbf{OTPrune (Ours)} & 1.002 (15.71) 
& \textbf{1.03} & \textbf{0.71} & 57.11 & \textbf{65.55} & \textbf{1411} & \textbf{37.00} 
& \textbf{1.01} & \textbf{54.56} & 79.39 & \textbf{73.82} & \textbf{62.08} & \textcolor{red}{\textbf{1.27}} & \textcolor{red}{\textbf{1.18}} \\[3pt]

& \multicolumn{15}{l}{\hspace{0.3em}\textcolor{gray}{\scriptsize\textit{Category: Adaptive Pruning Methods}}}\\[-2pt]
& \hspace{0.8em}$\text{PruMerge}^{\bigtriangleup}$~\cite{PruMerge} & Variable & 0.80 & 0.53 & 52.01 & 58.93 & 1256 & 36.56 & 0.77 & 49.15 & 64.36 & 72.53 & 56.10 & 3.64 & --\\
& \hspace{0.8em}DivPrune$^{*}$ & Variable & 0.94 & 0.59 & 56.09 & 61.77 & 1344 & 34.89 & 0.91 & 50.86 & 79.60 & 71.34 & 60.00 & 3.14 & --\\
\midrule

\multirow{7}{*}{\rotatebox[origin=c]{90}{\textbf{LLaVA 1.6-7B}}} 

& \textbf{Original} 
& \cellcolor[gray]{0.9}11.849 (100.00) 
& \cellcolor[gray]{0.9}1.00 & \cellcolor[gray]{0.9}0.68 & \cellcolor[gray]{0.9}64.28 & \cellcolor[gray]{0.9}67.01 
& \cellcolor[gray]{0.9}1520 & \cellcolor[gray]{0.9}36.44 
& \cellcolor[gray]{0.9}0.88 & \cellcolor[gray]{0.9}44.20 & \cellcolor[gray]{0.9}86.38 
& \cellcolor[gray]{0.9}70.15 & \cellcolor[gray]{0.9}70.16 & \cellcolor[gray]{0.9}-- & \cellcolor[gray]{0.9}-- \\[2pt]

& \multicolumn{15}{l}{\hspace{0.3em}\textcolor{gray}{\scriptsize\textit{Category: One-shot Pruning Methods}}}\\[-2pt]
& \hspace{0.8em}VTW~\cite{vtw}       & 1.318 (11.23) & 0.06 & 0.03 & 38.62 & 19.76 & 606  & 31.30 & 0.03 & 8.66  & 7.13  & 65.74 & 37.48 & 3.82 & 3.82\\
& \hspace{0.8em}FastV~\cite{fastv}   & 1.327 (11.30) & 0.06 & 0.03 & 38.79 & 20.36 & 619  & 32.56 & 0.04 & 8.80  & 7.78  & 65.49 & 37.62 & 3.18 & 3.18\\
& \hspace{0.8em}DivPrune             & 1.266 (10.79) & 0.89 & 0.61 & 58.69 & 63.49 & 1362 & 37.11 & 0.76 & 41.92 & \textbf{82.97} & \textbf{68.57} & 64.11 & 1.77 & 1.77\\
& \hspace{0.8em}\textbf{OTPrune (Ours)} & 1.266 (10.79)         & \textbf{0.89} & \textbf{0.64} & \textbf{60.28} & \textbf{63.75} & \textbf{1414} & \textbf{38.33} & \textbf{0.84} & \textbf{43.26} & 81.41 & 68.07 & \textbf{64.79} & \textcolor{red}{\textbf{1.23}} & \textcolor{red}{\textbf{1.23}}\\

\bottomrule
\end{tabular}
}
\vspace{-5pt}
\label{tbl:benchmark_comparison_image}
\end{table*}

\section{Empirical Analysis of the Algorithm}
\label{sec:analysisOfAlg}

To assess the empirical behavior of the proposed \textit{OTPrune} algorithm, we conduct controlled experiments using synthetic data. Specifically, we generate multiple datasets with varying numbers of samples and feature dimensions. For each dataset, we also vary the number of selected samples to simulate different pruning ratios. When both the total number of samples $n$ and the selected subset size $k$ are relatively small, exhaustive subset evaluation is tractable, allowing us to enumerate all possible combinations $\binom{m}{k}$. In larger settings, where exhaustive search becomes infeasible, we adopt a Monte Carlo strategy by randomly sampling 10 million candidate subsets to approximate the full search space.

In this analysis, rather than evaluating the approximate objective $\tilde{f}(\cdot)$, we directly evaluate distribution distance via the original objective function defined in Eq.~\ref{eq:w3}, i.e., ${f}(\cdot)$ and assess the performance of each method (OTPrune and DivPrune) using two metrics:  
(i) the \emph{win rate}, and  
(ii) the \emph{optimality ratio}, defined as $f(\tilde{C})/f(C^{*})$, where $\tilde{C}$ denotes the subset selected by the corresponding algorithm and $C^{*}$ represents the optimal subset.
The \emph{win rate} measures the proportion of all possible subset selection strategies whose objective value $f(C)$ is lower than that achieved by the given method:
\[
\text{WinRate} = 
\frac{|\{C \subseteq [m],\, |C|=k \;|\; f(\tilde{C}) > f(C)\}|}{\binom{m}{k}}.
\]
This metric reflects how frequently a method’s selected subset surpasses alternative strategies across the entire combinatorial space.  
It can also be interpreted probabilistically as the expected dominance probability under random subset selection, since the uniform random subset samples each $C$ with equal probability:
$
\text{WinRate} =
\Pr_{C \sim \text{Uniform}(\binom{[n]}{k})}
\big[ f(\tilde{C}) > f(C) \big].
$
Hence, the win rate provides an intuitive statistical interpretation of how consistently an algorithm dominates random selection, while the optimality ratio quantifies its proximity to the best possible subset. 

\begin{table}[htbp]
\centering
\caption{
Comparison between OTPrune and DivPrune on synthetic datasets with varying numbers of samples ($m$), feature dimensions ($d$), and selected subset sizes ($k$). 
The first two configurations are evaluated exhaustively over all $\binom{m}{k}$ combinations, while the last three configurations use Monte Carlo estimation with 10 million different subsets to approximate the full combinatorial space.
}
\label{tab:small-exp1}
\small
\resizebox{0.45\textwidth}{!}{%
\begin{tabular}{cccc|cc}
\toprule
\multicolumn{4}{c}{Configuration} & \multicolumn{2}{c}{Performance} \\ 
\midrule
$m$ & $d$ & $k$ & Method & Win Rate & Optimality Ratio \\ 
\midrule

\multirow{2}{*}{20} & \multirow{2}{*}{10} & \multirow{2}{*}{5} 
  & DivPrune & 60.14 & 74.53 \\
  & & & \cellcolor[gray]{0.9}\textbf{OTPrune} & \cellcolor[gray]{0.9}\textbf{94.56} & \cellcolor[gray]{0.9}\textbf{87.81} \\
\midrule

\multirow{2}{*}{30} & \multirow{2}{*}{20} & \multirow{2}{*}{10} 
  & DivPrune & 84.62 & 76.48 \\
  & & & \cellcolor[gray]{0.9}\textbf{OTPrune} & \cellcolor[gray]{0.9}\textbf{99.49} & \cellcolor[gray]{0.9}\textbf{93.98} \\
\midrule

\multirow{2}{*}{100} & \multirow{2}{*}{50} & \multirow{2}{*}{30} 
  & DivPrune & 88.77 & 92.20 \\
  & & & \cellcolor[gray]{0.9}\textbf{OTPrune} & \cellcolor[gray]{0.9}\textbf{99.99} & \cellcolor[gray]{0.9}\textbf{100.00} \\
\midrule

\multirow{2}{*}{100} & \multirow{2}{*}{80} & \multirow{2}{*}{50} 
  & DivPrune & 92.14 & 95.80 \\
  & & & \cellcolor[gray]{0.9}\textbf{OTPrune} & \cellcolor[gray]{0.9}\textbf{100.00} & \cellcolor[gray]{0.9}\textbf{100.00} \\
\midrule

\multirow{2}{*}{1000} & \multirow{2}{*}{256} & \multirow{2}{*}{100} 
  & DivPrune & 97.59 & 97.71 \\
  & & & \cellcolor[gray]{0.9}\textbf{OTPrune} & \cellcolor[gray]{0.9}\textbf{100.00} & \cellcolor[gray]{0.9}\textbf{100.00} \\
\bottomrule
\end{tabular}%
}
\end{table}

Table~\ref{tab:small-exp1} summarizes the results across different synthetic configurations. For each setting, we compare \textit{OTPrune} and \textit{DivPrune}, and include the implicit random selection baseline, whose expected performance equals the mean of $f(C)$ across all subsets. 
Both methods substantially outperform random selection, achieving high win rates and optimality ratios that confirm the effectiveness of structured token selection over naive sampling. 
However, OTPrune consistently achieves markedly higher scores than DivPrune, particularly in low-dimensional or small-sample regimes where diversity-only criteria are less reliable. 
Even when the search space becomes large, OTPrune maintains near-perfect performance (\textit{e.g.}, achieving 100\% win rate and 100\% optimality ratio in high-dimensional cases), demonstrating that our submodular approximation closely aligns with the original OT objective. 
Importantly, both methods empirically achieve optimality ratios exceeding the theoretical $(1 - 1/e)$ bound guaranteed for the maximization problem, indicating that the bound is conservative in practice. 
Nevertheless, OTPrune consistently attains ratios much closer to 1.0 across all settings, showing that its distribution-aware formulation more effectively captures global structure and achieves near-optimal performance. 
These results empirically confirm that modeling distributional alignment yields stronger representational coverage and robustness than diversity-only strategies.

\section{Experiment}
In this section, we will delve into a comprehensive comparison of our proposed method with prior works on downstream tasks.
Additionally, we give an analysis of the efficiency of OTPrune and include an ablation study to examine its key components.

\subsection{Experimental Settings}\label{sec:exp_setup}

\noindent\textbf{Datasets.} To rigorously evaluate our framework, we employ 11 distinct datasets covering a broad spectrum of multimodal reasoning and understanding benchmarks: ScienceQA-IMG (SQA)~\cite{sqa-i}, POPE~\cite{pope}, MME~\cite{mme}, MMB~\cite{mmb}, GQA~\cite{gqa}, MMMU~\cite{MMMU}, Flicker30k~\cite{flicker}, SeedBench (SEEDB)~\cite{SeedBench}, Nocaps~\cite{nocaps}, OKVQA~\cite{okvqa}, and COCO-2017~\cite{coco}. This diverse collection allows us to assess capabilities across varied domains, ranging from image captioning to both multiple-choice and open-ended visual question answering scenarios.

\noindent\textbf{Evaluation metrics.} Following established evaluation protocols~\cite{divprune}, we measure image captioning quality using the CIDEr metric~\cite{cider}. For question-answering paradigms, model outputs are quantified via Exact Match (EM), Accuracy (Acc), F1, Perception Score (P-score)~\cite{mme}, and GPT-assisted evaluations~\cite{gptscore}. Additionally, open-ended QA responses are judged utilizing the Wu-Palmer similarity (WUPS)~\cite{wups} alongside GPT-assisted scoring~\cite{gptscore}. Throughout our analysis, ascending values in these metrics consistently reflect superior predictive capability. To gauge efficiency, similar to previous studies~\cite{fastv,FitPrune,vtw}, computational cost is recorded in TFLOPs (aligning with \cite{divprune}), where lower numbers designate a more efficient process.

\noindent\textbf{Baselines and Models.} We compare OTPrune against several established token pruning techniques~\cite{divprune}, specifically FastV~\cite{fastv}, PruMerge~\cite{PruMerge}, VTW~\cite{vtw}, FitPrune~\cite{FitPrune}, $\text{M}^3$~\cite{m3}, and DivPrune~\cite{divprune}. Our primary analysis focuses on FastV, VTW, DivPrune, and PruMerge as our \textbf{{main competitors}}, as these offer plug-and-play functionality without the overhead of subsequent fine-tuning or dataset calibration. Within this group, FastV, VTW, and DivPrune operate under a \textit{one-shot} paradigm (utilizing a globally shared pruning ratio), while PruMerge functions as an \textit{adaptive} method (determining token retention dynamically per image). To ensure a comprehensive evaluation, we also provide comparative metrics against FitPrune (a calibration-dependent approach) and $\text{M}^3$ (a fine-tuning-based framework). It is important to highlight that VTW intrinsically relies on a calibration step to pinpoint the optimal pruning layer. Since this standard behavior prevents alignment with a fixed TFLOP target, we intentionally deactivate VTW's calibration phase in our tests. Instead, we manually select a layer that forces VTW to satisfy the specific FLOP constraints of the given experiment, ensuring a fair, resource-matched comparison.

\noindent\textbf{Implementation.} To illustrate the broad applicability and generality of OTPrune, we deploy our algorithm and competing baselines across multiple widely used LMM architectures, including LLaVA 1.5-7B~\cite{llava_1.5}\footnote{https://huggingface.co/liuhaotian/llava-v1.5-7B}, LLaVA 1.5-13B~\cite{llava_1.5}\footnote{https://huggingface.co/liuhaotian/llava-v1.5-13b}, and LLaVA 1.6-7B\footnote{https://huggingface.co/liuhaotian/llava-v1.6-vicuna-7b} (alternatively referred to as LLaVA-NeXT~\cite{llavanext}). In our tabulated results, we only display the baseline comparisons that are directly compatible with the specific model and task combination being evaluated. Architecturally, every LMM in our study incorporates the CLIP vision encoder~\cite{clip}. While the LLaVA 1.5 variant maps images into a fixed grid of 576 tokens, LLaVA 1.6 scales the patch extraction dynamically, yielding a token sequence approximately 3 to 5 times longer than its predecessor. Hardware-wise, all benchmarks were executed on a single 1 $\times$ A100 GPU with 80GB VRAM. The inference and evaluation pipelines were constructed using the lmms-evals package~\cite{lmmseval}, with tests running uniformly at a batch size of 1. Finally, for automated assessments dependent on the ChatGPT API, queries were processed via the ``gpt-4o-mini" engine.

\subsection{Main Results}

Table~\ref{tbl:benchmark_comparison_image} summarizes the quantitative comparison of OTPrune against existing pruning approaches across 11 multimodal benchmarks.
Among training-free one-shot baselines (VTW~\cite{lin2025boosting}, FastV~\cite{fastv}, and DivPrune~\cite{divprune}), OTPrune consistently achieves the top rank. For LLaVA 1.5-7B, OTPrune improves the average rank from 3.23 (DivPrune) to 2.36 and the one-shot rank from 1.64 to 1.36, yielding higher CIDEr and QA scores at roughly 15\% of the original FLOPs. On LLaVA 1.5-13B, the average rank drops from 2.14 to 1.27 and the one-shot rank to 1.18. For LLaVA 1.6-7B, OTPrune maintains the best performance (Avg. Rank = 1.23) at approximately 11\% computation. 
These improvements confirm that distribution-aware selection via optimal transport yields better semantic coverage than diversity-only pruning. Adaptive strategies such as PruMerge~\cite{PruMerge} and DivPrune* remain directly comparable, as they require no additional training. 
OTPrune achieves comparable or higher accuracy while maintaining consistent efficiency, indicating that distributional alignment can provide the same adaptivity benefits without explicit layer- or ratio-specific calibration.
Approaches such as FitPrune~\cite{FitPrune} and M\textsuperscript{3}~\cite{m3} involve additional training or external calibration datasets, making them not strictly comparable to our one-shot setting. 
Nevertheless, OTPrune attains similar or superior task performance without any retraining, underscoring its practicality as a fully training-free alternative.




\renewcommand{\arraystretch}{1.03} 

\noindent\textbf{Statistic Test.} To assess the consistency and significance of performance differences across multiple datasets and evaluation metrics, we conducted a non-parametric statistical analysis based on the Wilcoxon signed-rank test. Specifically, we computed the \textit{relative} percentage improvement of every metric with respect to its baseline configuration, i.e., Divprune. The Wilcoxon signed-rank test was then applied across all metrics to evaluate whether the median improvement was significantly greater than zero (one-sided test). 
\begin{wraptable}[4]{r}{0.5\linewidth}
\centering
\label{tab:sts_main}
\resizebox{0.99\linewidth}{!}{%
\begin{tabular}{c|c}\toprule
Comparison  & \textit{p}-value \\ \midrule
\textit{OTPrune} vs \textit{DivPrune} & 0.028 \\\bottomrule
\end{tabular}%
}
\end{wraptable}
As shown in the right table, our method achieves a \textit{p}-value of 0.028, indicating that it statistically outperforms {DivPrune}, the previous state-of-the-art.

In summary, OTPrune delivers the strongest accuracy--efficiency trade-off among all directly comparable pruning methods and remains competitive even against fine-tuned variants. 

\begin{figure*}
    \centering
    \includegraphics[width=1\linewidth]{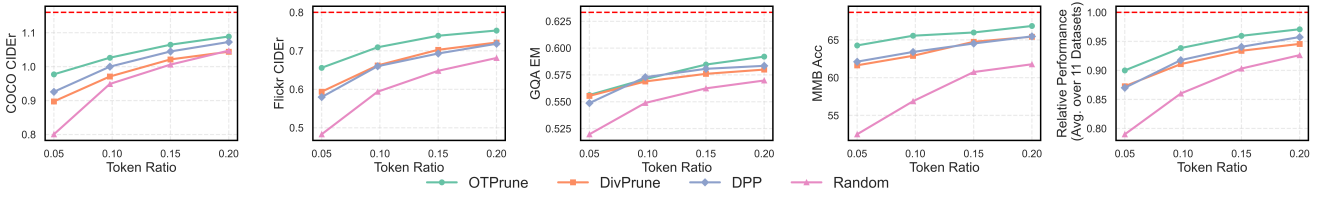}
    \caption{\textbf{Comparison with diversity-based methods using LLaVA~1.5-13B.} 
We evaluate \textit{OTPrune}, \textit{DivPrune}~\cite{divprune}, \textit{DPP}~\cite{kulesza2012determinantal}, and \textit{Random sampling} across 11 multimodal benchmarks under different pruning ratios $\{0.05,\,0.098,\,0.15,\,0.2\}$. 
We report both absolute performance and \textit{relative performance}, defined as the ratio between pruned and original model performance (\textit{pruned/original}). 
We also compute the OT distance between each subset and the full token set and visualize the \textit{relative distance} normalized by OTPrune, i.e., $f(\mathcal{C}_{\text{method}})/f(\mathcal{C}_{\text{OTPrune}})$. 
 }
    \label{fig:Performance_aba}
\end{figure*}

\begin{figure*}
    \centering
    \includegraphics[width=1\linewidth]{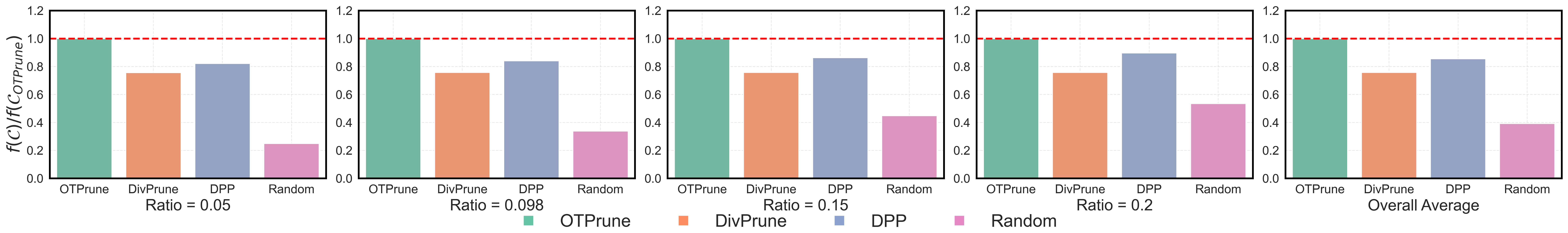}
    \caption{\textbf{Relative OT distance comparison across pruning methods.} 
We evaluate \textit{OTPrune}, \textit{DivPrune}~\cite{divprune}, \textit{DPP}~\cite{kulesza2012determinantal}, and \textit{Random sampling} using LLaVA~1.5-13B across 11 multimodal benchmarks under different token ratios $\{0.05,\,0.098,\,0.15,\,0.2\}$. 
Each bar shows the relative OT distance normalized by OTPrune, i.e., $f(\mathcal{C}_{\text{method}})/f(\mathcal{C}_{\text{OTPrune}})$, where lower values indicate closer alignment with the original token distribution. 
\textbf{OTPrune consistently achieves the smallest OT distance across all ratios, demonstrating superior distributional fidelity. 
The last panel reports the overall average across all 11 datasets.}}
    \label{fig:Relative_Performance_to_OTPrune_with_Avg}
\end{figure*}

\begin{figure}
    \centering
    \includegraphics[width=0.9\linewidth]{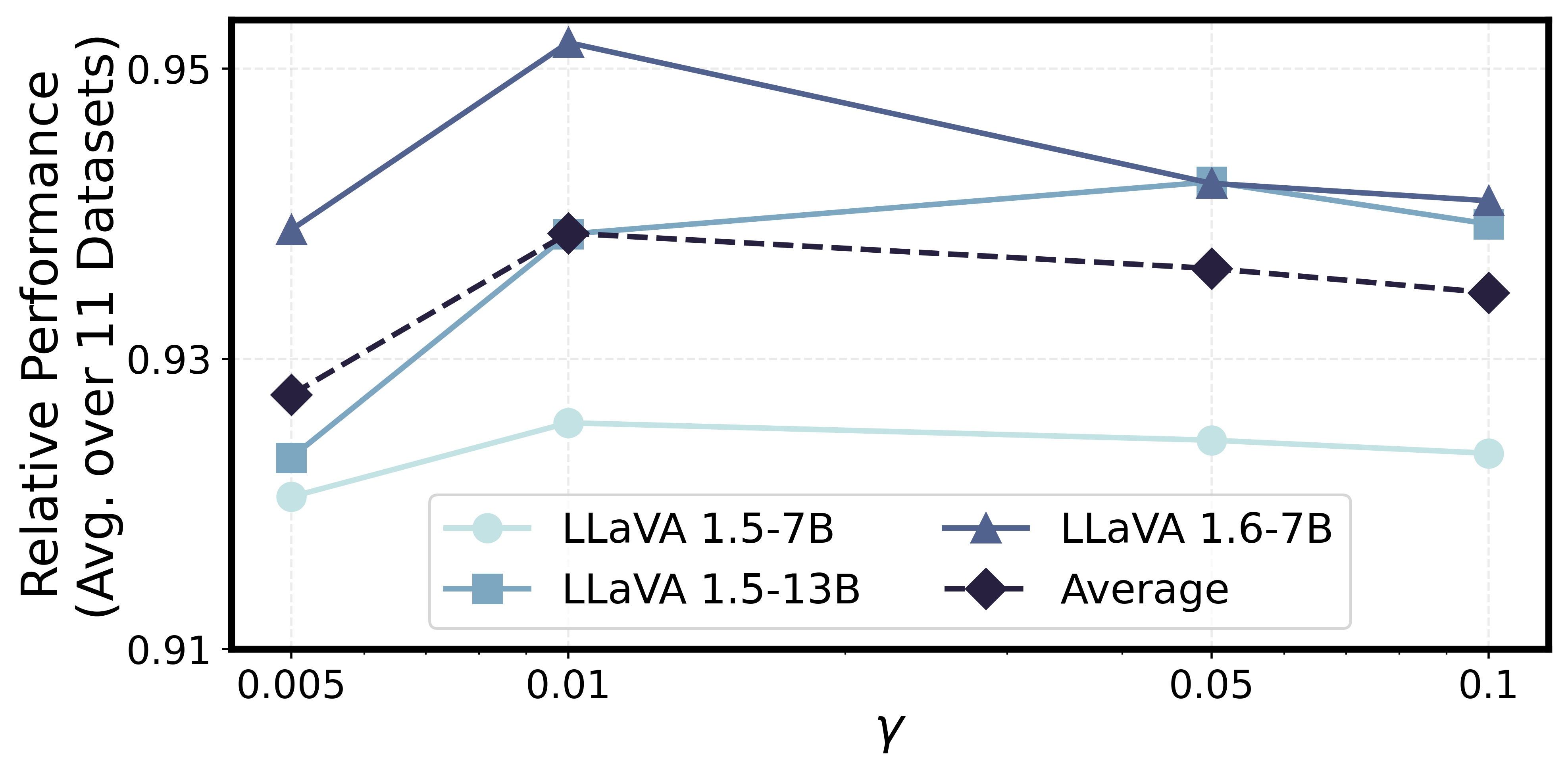}
    \caption{Sensitivity analysis of the balancing coefficient $\gamma$.
}
    \label{fig:pruning_relative_performance_curve_sns}
\end{figure}

\subsection{Analysis}

\noindent\textbf{Comparison with Diversity-based Methods.} 
For a comprehensive evaluation, we conduct experiments using LLaVA~1.5-13B across 11 multimodal benchmarks with varying token ratios $\{0.05,\,0.098,\,0.15,\,0.2\}$. 
In addition to \textit{OTPrune} and the state-of-the-art \textit{DivPrune}~\cite{divprune}, we include the widely used diversity-based selection method, Determinantal Point Process (\textit{DPP})~\cite{kulesza2012determinantal}, as well as \textit{Random sampling}. 
We report both the absolute task performance for each dataset and the \textit{relative performance}, defined as the ratio between pruned and original model performance (\textit{pruned/original}), to account for scale differences across metrics. 
To further validate our hypothesis, we compute the OT distance between each selected subset and the full token set. 
For readability, we present the \textit{relative distance} normalized by OTPrune, i.e., $f(\mathcal{C}_{\text{method}})/f(\mathcal{C}_{\text{OTPrune}})$. 
As shown in Fig.~\ref{fig:Performance_aba} and Fig.~\ref{fig:Relative_Performance_to_OTPrune_with_Avg}, OTPrune consistently achieves the lowest OT distance and highest relative performance across all token ratios, with clear advantages that persist even in the averaged results over 11 datasets.

\noindent\textbf{Sensitivity of $\gamma$.} 
We further investigate the sensitivity of the balancing coefficient $\gamma$ in Eq.~(\ref{eq:new_objective}) that controls the trade-off between the OT loss and the token selection regularization. 
We experiment with $\gamma \in \{0.005,\, 0.01,\, 0.05,\, 0.1\}$ using LLaVA~1.5-7B, LLaVA~1.5-13B, and LLaVA~1.6-7B across 11 multimodal benchmarks. 
As shown in Fig.~\ref{fig:pruning_relative_performance_curve_sns}, the overall trend remains stable across models, demonstrating that OTPrune is robust to the choice of $\gamma$. 
Performance slightly improves when $\gamma$ increases from 0.005 to 0.01, indicating a beneficial balance between distributional alignment and sparsity, and then plateaus for larger values. 
This suggests that OTPrune does not require extensive hyperparameter tuning, and $\gamma = 0.01$ serves as a good default across model scales.

\section{Conclusion}
In this work, we proposed OTPrune, a principled token pruning framework formulated as an optimal transport problem.
Unlike prior task-specific or heuristic pruning approaches, OTPrune aligns the distribution of pruned tokens with that of the full token set, offering a task-agnostic and geometry-aware objective.
We show that minimizing the Wasserstein distance effectively preserves the representational structure of visual-language models.
Extensive experiments on 11 multimodal reasoning and understanding benchmarks validate that better distributional approximation yields stronger performance.
OTPrune consistently outperforms diversity-based and random baselines across pruning ratios and model scales.
Our analysis further shows that OT distance correlates strongly with downstream accuracy, providing an interpretable metric for pruning quality.
The method is robust to the trade-off parameter $\gamma$, requiring minimal tuning.
Moreover, OTPrune generalizes well across LLaVA variants, confirming its scalability and broad applicability.
We hope this work inspires further research on distribution-aware pruning that improves efficiency without compromising performance in MLLMs.

\section*{Acknowledgments}
\label{sec:Acknowledgement}

This material is based upon the work supported by the National Science Foundation under Grant Number CNS-2204721.

{
    \small
    \bibliographystyle{ieeenat_fullname}
    \bibliography{main}
}


\appendix
\section{Proofs}

\subsection{Proof of Inequality}
Recap, for a positive semidefinite square matrix $\boldsymbol{X}$, the operator $\Psi(\boldsymbol{\cdot})$ is presented as:
\begin{equation}
\Psi(\boldsymbol{X}) = \log\det(\boldsymbol{I} + \gamma \boldsymbol{X})
= \sum_{i=1}^{d} \log(1 + \gamma \lambda_i),
\end{equation}
where $\{\lambda_i\}$ are the eigenvalues of $\boldsymbol{X}$ and $\gamma > 0$ is a fixed scaling parameter. 
Since $\boldsymbol{X}$ is positive semidefinite (so $\lambda_i \ge 0$), we have the following inequalities,
\begin{equation}
\Psi(\boldsymbol{X}) \le \gamma\,\operatorname{tr}(\boldsymbol{X})\leq \gamma (\operatorname{tr}(\boldsymbol{X}^{1/2}))^2.
\end{equation}
Now, we will show its proof as follows.

\begin{proof}
 Let $\boldsymbol{X}\succeq \boldsymbol{0}$ with non-negative eigenvalues $\{\lambda_i\}$.
For the first inequality in~\eqref{eq:approx}, use the elementary bound $\log(1+t)\le t$ for all $t\ge 0$. 
Since $\gamma\lambda_i\ge 0$,
\[
\log(1+\gamma\lambda_i)\le \gamma\lambda_i
\quad\Rightarrow\quad
\Psi(\boldsymbol{X}) \le \gamma\sum_{i=1}^d \lambda_i
= \gamma\,\operatorname{tr}(\boldsymbol{X}).
\]

For the second inequality in~\eqref{eq:approx}, note that
\begin{align*}
\bigl(\operatorname{tr}(\boldsymbol{X}^{1/2})\bigr)^2
&= \Bigl(\sum_{i=1}^d \sqrt{\lambda_i}\Bigr)^2 \\
&= \sum_{i=1}^d \lambda_i + 2\sum_{1\le i<j\le d} \sqrt{\lambda_i\lambda_j} \\
&\;\ge\; \sum_{i=1}^d \lambda_i 
= \operatorname{tr}(\boldsymbol{X}),
\end{align*}
because each term $\sqrt{\lambda_i\lambda_j}$ is nonnegative. Multiplying both sides by $\gamma>0$ gives
\[
\gamma\,\operatorname{tr}(\boldsymbol{X})
\;\le\;
\gamma\bigl(\operatorname{tr}(\boldsymbol{X}^{1/2})\bigr)^2.
\]

Combining the two inequalities yields
\[
\Psi(\boldsymbol{X})
\;\le\; \gamma\,\operatorname{tr}(\boldsymbol{X})
\;\le\; \gamma\bigl(\operatorname{tr}(\boldsymbol{X}^{1/2})\bigr)^2,
\]
as claimed.
\end{proof}

\subsection{Proof of Submodularity}

\begin{definition}[Submodularity]
A set function $h: 2^{[m]} \to \mathbb{R}$ is \emph{submodular} if it satisfies
the diminishing-returns property: for any $\mathcal{S}\subseteq\mathcal{T}\subseteq [m]$
and any $i\notin\mathcal{T}$,
\[
h(\mathcal{S}\cup\{i\}) - h(\mathcal{S})
\;\ge\;
h(\mathcal{T}\cup\{i\}) - h(\mathcal{T}).
\]
\end{definition}

\begin{lemma}\label{lemma:1}
(\cite{kulesza2012determinantal,krause2007near,bach2013learning}) Let $\boldsymbol{A} \in \mathbb{R}^{m \times m}$ be a positive definite matrix, 
and for any subset $S \subseteq \{1,\dots,m\}$ let $\boldsymbol{A}_S$ denote the 
corresponding principal submatrix. Define the set function
\[
    h(S) = \log \det(\boldsymbol{A}_S).
\]
Then $h$ is a submodular set function.
\end{lemma}

\begin{prop}
 Let $\boldsymbol{V}\in\mathbb{R}^{m\times d}$ and let $\boldsymbol{V}_{\mathcal{C}}$ denote
the submatrix of $\boldsymbol{V}$ consisting of rows indexed by $\mathcal{C}\subseteq[m]$.
Define
\[
\Gamma(\mathcal{C})
=
\log \det\!\Big(
\boldsymbol{I}
+ \tilde{\gamma}\,
\boldsymbol{V}_{\mathcal{C}}
\boldsymbol{V}^\top
\boldsymbol{V}
\boldsymbol{V}_{\mathcal{C}}^\top
\Big),
\qquad \tilde{\gamma}>0.
\]
Then $\Gamma$ is normalized and submodular.   
\end{prop}

\begin{proof}
To prove this proposition, We need to first prove this:
Let $\boldsymbol{V} \in \mathbb{R}^{m \times d}$ be any matrix and 
$\mathcal{C} \subseteq [m]$ be an index set. Let $\boldsymbol{V}_{\mathcal{C}}$ 
denote the submatrix of $\boldsymbol{V}$ formed by the rows indexed by 
$\mathcal{C}$, and let $\tilde{\gamma} \ge 0$. Then the matrix
\[
    \boldsymbol{M}
    =
    \boldsymbol{I}_{|\mathcal{C}|}
    + \tilde{\gamma}\,
      \boldsymbol{V}_{\mathcal{C}}
      \boldsymbol{V}^\top
      \boldsymbol{V}
      \boldsymbol{V}_{\mathcal{C}}^\top
    \in \mathbb{R}^{|\mathcal{C}| \times |\mathcal{C}|}
\]
is symmetric positive definite.

Since $\boldsymbol{V}^\top \boldsymbol{V} \in \mathbb{R}^{d \times d}$ is 
symmetric positive semidefinite, define
\[
    \boldsymbol{A}
    =
    \boldsymbol{V}_{\mathcal{C}}
    \boldsymbol{V}^\top
    \boldsymbol{V}
    \boldsymbol{V}_{\mathcal{C}}^\top
    \in \mathbb{R}^{|\mathcal{C}| \times |\mathcal{C}|}.
\]
We first show that $\boldsymbol{A}$ is positive semidefinite. For any 
$\boldsymbol{x} \in \mathbb{R}^{|\mathcal{C}|}$,
\[
    \boldsymbol{x}^\top \boldsymbol{A} \boldsymbol{x}
    =
    \boldsymbol{x}^\top
    \boldsymbol{V}_{\mathcal{C}}
    \boldsymbol{V}^\top
    \boldsymbol{V}
    \boldsymbol{V}_{\mathcal{C}}^\top
    \boldsymbol{x}
    =
    \bigl(\boldsymbol{V}_{\mathcal{C}}^\top \boldsymbol{x}\bigr)^\top
    \boldsymbol{V}^\top \boldsymbol{V}
    \bigl(\boldsymbol{V}_{\mathcal{C}}^\top \boldsymbol{x}\bigr).
\]
Let $\boldsymbol{y} = \boldsymbol{V}_{\mathcal{C}}^\top \boldsymbol{x} \in \mathbb{R}^d$. Then
\[
    \boldsymbol{x}^\top \boldsymbol{A} \boldsymbol{x}
    =
    \boldsymbol{y}^\top \boldsymbol{V}^\top \boldsymbol{V} \boldsymbol{y}
    =
    \|\boldsymbol{V}\boldsymbol{y}\|_2^2
    \ge 0,
\]
so $\boldsymbol{A}$ is positive semidefinite.

Now consider
\[
    \boldsymbol{M}
    =
    \boldsymbol{I}_{|\mathcal{C}|}
    + \tilde{\gamma}\, \boldsymbol{A}.
\]
The matrix $\boldsymbol{M}$ is symmetric because both 
$\boldsymbol{I}_{|\mathcal{C}|}$ and $\boldsymbol{A}$ are symmetric.
For any nonzero $\boldsymbol{x} \in \mathbb{R}^{|\mathcal{C}|}$ we have
\[
    \boldsymbol{x}^\top \boldsymbol{M} \boldsymbol{x}
    =
    \boldsymbol{x}^\top \boldsymbol{I}_{|\mathcal{C}|} \boldsymbol{x}
    + \tilde{\gamma}\, \boldsymbol{x}^\top \boldsymbol{A} \boldsymbol{x}
    =
    \|\boldsymbol{x}\|_2^2
    + \tilde{\gamma}\, \boldsymbol{x}^\top \boldsymbol{A} \boldsymbol{x}.
\]
Since $\boldsymbol{A}$ is positive semidefinite and $\tilde{\gamma} \ge 0$, we have
$\boldsymbol{x}^\top \boldsymbol{A} \boldsymbol{x} \ge 0$, so
\[
    \boldsymbol{x}^\top \boldsymbol{M} \boldsymbol{x}
    \ge
    \|\boldsymbol{x}\|_2^2
    > 0
\]
for all $\boldsymbol{x} \neq \boldsymbol{0}$. Therefore $\boldsymbol{M}$ is 
positive definite.

Then we can apply Lemma~\ref{lemma:1}, and we can conclude $\Gamma$ is sumdoular.

\end{proof}

\section{Detail of Algorithm}

In this appendix we describe how our pruning objective naturally leads to the OTPrune algorithm.  
The derivation begins directly with the token--token interaction matrix and proceeds to the incremental update rules used in Algorithm~1, without introducing additional notation beyond what is necessary.

Given the vision token matrix $\boldsymbol V \in \mathbb R^{m\times d}$,  
we first construct a token--token interaction matrix
\[
\tilde{\boldsymbol V} = \boldsymbol V \boldsymbol V^\top \in \mathbb R^{m\times m}.
\]
The $i$-th row of this matrix,
\[
\boldsymbol w_i^\top = \tilde{\boldsymbol V}[i,:],
\]
captures how token $i$ interacts with all other tokens.  
For a subset $\mathcal C \subseteq [m]$, we use
\[
\tilde{\boldsymbol V}_{\mathcal C} = \tilde{\boldsymbol V}[\mathcal C,:]
\]
to denote the corresponding row submatrix.

Our pruning objective is the log-determinant
\[
\Gamma(\mathcal C)
=
\log \det\!\Big(
\boldsymbol I
+
\tilde{\gamma}\,
\tilde{\boldsymbol V}_{\mathcal C}
\tilde{\boldsymbol V}_{\mathcal C}^\top
\Big),
\qquad
\tilde{\gamma}>0.
\]
This quantity becomes large when the selected interaction vectors 
$\{\boldsymbol w_i : i\in\mathcal C\}$ span a diverse subspace.  
It is convenient to define the kernel
\[
K
=
\boldsymbol I
+
\tilde{\gamma}\,
\tilde{\boldsymbol V}\,
\tilde{\boldsymbol V}^\top,
\quad\text{so that}\quad
K_{ij}
=
\delta_{ij}
+
\tilde{\gamma}\,
\langle \boldsymbol w_i,\boldsymbol w_j\rangle.
\]
Here, the Kronecker delta $\delta_{ij}$ equals $1$ if $i=j$ and $0$ otherwise.
Then $\Gamma(\mathcal C)=\log\det(K_{\mathcal C})$,  
where $K_{\mathcal C}$ is the principal submatrix indexed by $\mathcal C$.

\vspace{1em}
\noindent\textbf{Greedy marginal gain.}
When considering a new token $j\notin\mathcal C$, the determinant of the enlarged matrix can be written using the classical one-step determinant identity:
\[
\det(K_{\mathcal C\cup\{j\}})
=
\det(K_{\mathcal C})\cdot
\Big(
K_{jj}
-
K_{\mathcal C j}^\top
K_{\mathcal C}^{-1}
K_{\mathcal C j}
\Big).
\]
Thus the improvement in the objective is
\[
\Gamma(\mathcal C\cup\{j\})-\Gamma(\mathcal C)
=
\log d_j^2,
\qquad
d_j^2
=
K_{jj}
-
K_{\mathcal C j}^\top
K_{\mathcal C}^{-1}
K_{\mathcal C j}.
\]
The greedy algorithm simply selects the token with the largest $d_j^2$.

\vspace{1em}
\noindent\textbf{Cholesky representation.}
Let the Cholesky factor of the current kernel submatrix be
\[
K_{\mathcal C} = LL^\top.
\]
Define the coefficient vector
\[
\boldsymbol c_j = L^{-1} K_{\mathcal C j}.
\]
Because $K_{\mathcal C}^{-1}=(L^{\top})^{-1} L^{-1}$, the quadratic form simplifies to
\[
K_{\mathcal C j}^\top K_{\mathcal C}^{-1} K_{\mathcal C j}
=
\|\boldsymbol c_j\|^2,
\]
so the greedy score becomes
\[
d_j^2 = K_{jj} - \|\boldsymbol c_j\|^2.
\]
Thus maintaining the greedy scores reduces to tracking $\boldsymbol c_i$ and $d_i^2$ for all unselected tokens.

\vspace{1em}
\noindent\textbf{Adding a token.}
Suppose token $j$ is selected.  
The Cholesky factor of the enlarged set $\mathcal C\cup\{j\}$ takes the form
\[
L' =
\begin{pmatrix}
L & 0 \\
\boldsymbol c_j^\top & d_j
\end{pmatrix}.
\]
For any remaining token $i\notin\mathcal C\cup\{j\}$,  
let the updated coefficient vector be
\[
\boldsymbol c_i' =
\begin{bmatrix}
\boldsymbol c_i \\[3pt]
e_i
\end{bmatrix}.
\]
Plugging this into $L'\boldsymbol c_i' = K_{\mathcal C\cup\{j\},\, i}$ yields
\[
e_i
=
\frac{
K_{ji}
-
\langle \boldsymbol c_j,\boldsymbol c_i\rangle
}{d_j}.
\]
Once this new coordinate is known, the updated greedy score becomes
\[
d_i'^2
=
K_{ii} - \|\boldsymbol c_i'\|^2
=
d_i^2 - e_i^2.
\]

In general, the update takes the form
\[
e_i = \frac{K_{ji} - \langle \boldsymbol c_j,\boldsymbol c_i\rangle}{d_j}.
\]
For our kernel $K_{ji} = \delta_{ji} + \tilde{\gamma}\langle \boldsymbol w_j,\boldsymbol w_i\rangle$.
Since we only update $i \neq j$ (token $j$ has just been selected), the Kronecker
delta term vanishes and we obtain
\[
e_i = \frac{\tilde{\gamma}\langle \boldsymbol w_j,\boldsymbol w_i\rangle
- \langle \boldsymbol c_j,\boldsymbol c_i\rangle}{d_j}.
\]

This expression exactly the operations executed in Algorithm~1.

\vspace{1em}
\noindent\textbf{Initialization and complexity.}
With $\mathcal C=\emptyset$, all coefficient vectors are empty and the initial
scores reduce to
\[
d_i^2 = K_{ii} = 1 + \tilde{\gamma}\,\|\boldsymbol w_i\|^2.
\]
At iteration $t$, each of the $(m-t)$ remaining items requires an
$O(t)$ inner product $\langle \boldsymbol c_j,\boldsymbol c_i\rangle$,
so the overall cost of the iteration is
\[
O\big((m-t)t\big).
\]
Summing over $t=1,\dots,k$ gives
\[
\sum_{t=1}^k O\big((m-t)t\big) = O(mk^2),
\]
which is the total complexity of selecting $k$ tokens.

This completes the derivation of OTPrune directly from the log-determinant objective.

\section{Discussion on Non-uniform Weighting Strategies}
\label{sec:appendix_non_uniform}

In our current framework, we adopt a \textit{uniform weight assumption} for the source distribution $P$, where each token $v_i \in \boldsymbol{V}$ is assigned an equal mass of $1/M$. While this simplification allows for a highly efficient and training-free selection process, it treats all visual regions as equally significant, which may not fully capture the semantic variance inherent in complex images.

\textit{Potential for Adaptive Importance.} Theoretically, the optimal transport formulation in \textit{OTPrune} can be generalized to accommodate \textit{non-uniform weights} by assigning an importance score $w_i$ to each token, such that $\sum w_i = 1$. As suggested during the rebuttal, such weights could be derived from internal model signals to better reflect semantic density:
\begin{itemize}
    \item \textit{Attention-based Priors}: Leveraging the cumulative attention weights from preceding transformer layers to identify patches that the model naturally prioritizes.
    \item \textit{Task-specific Significance}: Utilizing cross-modal alignment scores to assign higher weights to visual tokens that are most relevant to the input text prompt.
\end{itemize}

\noindent\textit{Future Challenges.} Transitioning to a non-uniform weighting strategy transforms \textit{OTPrune} into a more complex joint optimization problem involving both subset selection and mass allocation. We consider the exploration of \textit{soft-weighted OTPrune}, where token weights are dynamically adjusted during the pruning process—as a promising direction for further improving performance at extremely low pruning ratios.

\end{document}